\documentclass{elsarticle}

\usepackage{amssymb}
\usepackage{enumitem}
\usepackage{subfigure}
\usepackage{float}
\floatstyle{plaintop}
\restylefloat{table}
\usepackage{hhline}
\usepackage{booktabs}
\usepackage[english]{babel}
\usepackage{threeparttable}
\usepackage{verbatim}
\usepackage{xfrac}
\usepackage{amsmath,amssymb}

\usepackage{breqn}

\usepackage{amsthm}

\newcommand{\pop}{{\rm{Popular}}}

\newtheorem{theorem}{Theorem}
\newtheorem{lemma}{Lemma}
\newtheorem{corollary}{Corollary}

\begin{document}

\title{Exact Learning of Juntas from Membership Queries}

\author{Nader H. Bshouty}
\ead{bshouty@cs.technion.ac.il}

\author{Areej Costa\corref{cor1}}
\ead{areej.costa@cs.technion.ac.il}

\address{Department of Computer Science, Technion, 3200003 Haifa, Israel}

\cortext[cor1]{Corresponding author}


\newcommand{\N}[1]{^{\mbox{\tiny(#1)\ }}}
\newcommand{\E}{{\bf E}}

  \newcommand{\specialcell}[2][l]{%
  \begin{tabular}[#1]{@{}l@{}}#2\end{tabular}}

\newcommand{\red}[1]{\textcolor{red}{{#1}}}
\newcommand{\ignore}[1]{}
\newcommand{\ignoreneed}[1]{}


\begin{abstract} In this paper we study adaptive and non-adaptive exact learning of Juntas from membership queries. We use new techniques to find new bounds, narrow some of the gaps between the lower bounds and upper bounds and find new deterministic and randomized algorithms with small query and time complexities.

Some of the bounds are tight in the sense that finding better ones either gives a breakthrough result in some long-standing combinatorial open problem or needs a new technique that is beyond the existing ones.
\end{abstract}

\begin{keyword}
	Exact learning \sep Membership queries \sep Non-adaptive learning \sep Adaptive learning \sep Juntas \sep Relevant variables
\end{keyword}	

\maketitle
\section{Introduction}

Learning from membership queries, \cite{A87}, has flourished due to its many applications in group testing, blood testing, chemical leak testing, chemical reactions, electric shorting detection, codes, multi-access channel communications, molecular biology, VLSI testing and AIDS screening. See many other applications in~\cite{DH00,ND00,BGV05,DH06,Ci13,BG07}. Many of the new applications raised new models and new problems and in many of those applications the function being learned can be an arbitrary function that depends on few variables. We call this class of functions $d$-Junta, where $d$ is a bound on the number of relevant variables in the function.
In some of the applications non-adaptive algorithms are most desirable, where in others adaptive algorithms with limited number of rounds are also useful. Algorithms with high number of rounds can also be useful given that the number of queries they ask is low. In all of the applications, one searches for an algorithm that runs in polynomial time and asks as few queries as possible. In some applications asking queries is very expensive, and therefore, even improving the query complexity by a small non-constant factor is interesting.

In this paper we study adaptive and non-adaptive exact learning of Juntas from membership queries. This problem was studied in~\cite{D00,D98,D03,BGV05}. In this paper, we find new bounds, tighten some of the gaps between some lower bounds and upper bounds\footnote{The term \emph{upper bound} refers to learning with unlimited computational power.} and find new algorithms with better query and time complexities both for the deterministic and the randomized case.

Since learning one term of size $d$ (which is a function in $d$-Junta) requires at least $2^d$ queries and asking one query requires time $O(n)$, we cannot expect to learn the class $d$-Junta in time less than $\Omega(2^d+n).$
We say that the class $d$-Junta is polynomial time learnable if there is an algorithm that learns $d$-Junta in time $poly(2^d,n)$. In this paper we also consider algorithms that run in time $n^{O(d)}$, which is polynomial time for constant $d$, and algorithms that run in time $poly(d^d,n)$, which is polynomial time for $d =O(\log n/\log\log n)$.


\subsection{Results for Non-adaptive Learning}
A set $S\subseteq \{0,1\}^n$ is called an $(n,d)$-universal set if for every $1\le i_1<i_2<\cdots<i_d\le n$
and $\sigma\in\{0,1\}^d$ there is an $s\in S$ such that $s_{i_j}=\sigma_j$ for all $j=1,\ldots,d$.
Damaschke,~\cite{D00}, shows that any set of assignments that learns $d$-Junta must be an $(n,d)$-universal set. This, along with the lower bound in~\cite{SB,KS72}, gives result (1) in Table~\ref{Non-adaptiveTable} for the deterministic case. It is not clear that this lower bound is also true for the randomized case. We use the minimax technique,~\cite{MR95}, to show that randomization cannot help reducing the lower bound. See (2) in the table.
\begin{center}
\begin{table}[H]

\begin{tabular}{c||c|c|c|c|c}
  \toprule
    & \specialcell{Lower\\Bound} & \specialcell{Upper\\Bound} & \specialcell{Time\\$n^{O(d)}$} &  \specialcell{Time\\$poly(d^d,n)$} & \specialcell{Time\\$poly(2^d,n)$} \\
  \hline  \hline
   \multicolumn{6}{c}{{\bf Deterministic}}\\
   \hline  \hline
   Previous  & $\N{1}2^{d}\log{n}$ & \specialcell{$\N{3}d2^{d}\log{n}$\\ \hspace{.25in}$+d^22^d$} & $\N{6}d^22^{d}\log{n}$ & --- &\specialcell{$\N{11}2^{d+O(\log^{2}{d})}\cdot$\\ \hspace{.50in}$\log^{2}{n}$}\\
   \hline
   Ours    & --- & $\N{4}d2^{d}\log{n}$ & $\N{7}d2^{d}\log{n}$ & $\N{9}d^32^d\log{n}$ & \specialcell{$\N{12}2^{d+O(\log^{2}{d})}\cdot$\\ \hspace{.50in}$\log{n}$}\\
   \hline  \hline
   \multicolumn{6}{c}{{\bf Randomized}}\\
   \hline  \hline
   Previous & --- & --- & --- & --- & ---    \\
   \hline
   Ours &$\N{2}2^d\log{n}$ & $\N{5}d2^{d}\log{n}$ & $\N{8}d2^{d}\log{n}$ & $\N{10}d^32^d\log{n}$ & \specialcell{$\N{13}(\log{d})d^22^d\cdot$\\ \hspace{.15in}$(\log n+\log{\frac{1}{\delta}})$ }\\
  \bottomrule
\end{tabular}
\caption{Results for Non-adaptive Learning}
\label{Non-adaptiveTable}
\end{table}
\end{center}


Damaschke introduces a graph-theoretic characterization of non-adaptive learning families, called $d$-wise bipartite connected families, \cite{D00}. He shows that for an $(n,d)$-universal set of assignments $S$, it can non-adaptively learn $d$-Junta if and only if $S$ is a $d$-wise bipartite connected family.
He then shows, with a non-constructive probabilistic method, that there
exists such a family of size $O(d2^d\log n+d^22^d)$. See (3) in the table. Then, he shows that a $d$-wise bipartite connected family of size $O(d^22^d\log n)$ can be constructed in $n^{O(d)}$ time, \cite{D00}. This, along with his deterministic learning algorithm presented in~\cite{D00} gives result (6) in the table. We further investigate the $d$-wise bipartite connected families and show that there exists one
of size $O(d2^d\log n)$. We then use the technique in~\cite{NSS95} to construct one of such size in time~$n^{O(d)}$.
This gives results (4), (5), (7) and (8), where the results for the randomized algorithms follow from the corresponding deterministic algorithms.
We then use the reduction of Abasi et. al in~\cite{ABM} to give a non-adaptive algorithm that runs in $poly(d^d,n)$ time and asks $O(d^32^d\log n)$ queries.
This is result (9) in the table. Then, result (10) for the randomized case follows.

We also introduce a new simple non-adaptive learning algorithm and apply the same reduction to this algorithm to get result (12) in the table.
This reduces the quadratic factor of $\log^2{n}$ in the query complexity of result (11) to a linear factor of $\log{n}$. Result (11) follows from a polynomial time learning algorithm for $d$-Junta given by Damaschke in \cite{D98}, which asks $2^{d+O(\log^2d)}\log^2 n$ queries. Finally, we give a new Monte Carlo randomized polynomial time algorithm that learns $d$-Junta with $O(d2^d(\log d+\log(1/\delta))\log n+d2^d(d+\log(1/\delta))(\log d+\log(1/\delta)))$ queries. Then we present a new reduction for randomized non-adaptive algorithms and apply it to this algorithm to get result (13) in the table. In addition, for $\delta \leq 1/d$, we give a Monte Carlo randomized polynomial time algorithm that asks $O((d2^d\log n+d^22^d)\log(1/\delta))$ queries. This algorithm is based on the previous one before the reduction. For values of $\delta$ small enough, the result in (13) is better.

The significant improvements over the previous results (see the table) are results (2), (4), (9), (12) and~(13).


\subsection{Results for Adaptive Learning}
As stated above, Damaschke shows in~\cite{D00} that any deterministic learning algorithm for $d$-Junta must ask a set of queries that is an $(n,d)$-universal set. This applies both for non-adaptive and adaptive learning. Therefore, applying the lower bound from~\cite{SB,KS72} gives result (1) in Table~\ref{AdaptiveTable}. A lower bound for the randomized case follows from a general lower bound presented in~\cite{ABM14}. See result (2). Damaschke,~\cite{D00}, shows that functions from $d$-Junta can be adaptively learned by asking a set of queries that is an $(n,d)$-universal set followed by $O(d\log{n})$ queries. This, along with the upper bound from the probabilistic method on the size of $(n,d)$-universal sets gives result (3) in the table. Furthermore, Damaschke's algorithm runs in time $T+poly(2^d,n)$, where $T$ is the construction time of the ${(n,d)}$-universal set. An ${(n,d)}$-universal set of size ${O(d2^{d}\log{n})}$ can be constructed deterministically by the derandomization method in $n^{O(d)}$ time as stated in \cite{NSS95}. This gives us an algorithm for adaptive learning that runs in ${n^{O(d)}}$ time and asks $O(d2^{d}\log{n})$ queries. See result (5). The polynomial time construction of an ${(n,d)}$-universal set in~\cite{NSS95} gives result (11) in the table.

Let $H$ be a family of functions ${h : [n] \rightarrow [q]}$, and let be ${d \leq q}$. We say that $H$ is an $(n,q,d)$-perfect hash family ($(n,q,d)$-PHF) if for every subset ${S \subseteq [n]}$ of size ${|S| = d}$ there is a hash function ${h \in H}$ such that ${|h(S)| = d}$. In~\cite{D98} Damaschke gives a two round algorithm that learns $d$-Junta in ${poly(d^d,n)}$ time using an $(n,d^2,d)$-PHF. This is result (7) in the table. We improve this algorithm by using another construction of PHF based on a lemma from~\cite{B15}. We can further improve the query complexity of this algorithm and get result (8) in the table based on a non-adaptive algorithm of ours.
A two round Monte Carlo randomized algorithm that runs in ${poly(d^d,n)}$ time is introduced in \cite{D98,D03}. If we modify its analysis slightly to make it applicable for a general failure probability $\delta$, we get the result in (9). In the paper we show how we can improve this query complexity and get result (10) in the table.

Results (4), (6) and (12) in the table are based on a new Monte Carlo randomized algorithm that we present in the paper. Result (13) follows also from a new Monte Carlo randomized algorithms that we introduce in this paper.

We note here that the results in (4), (6), (10), (12) and (13) are for values of the failure probability $\delta$ that are at most $1/d$. For $\delta > 1/d$, these three algorithms run in different query complexities as stated in Theorems~\ref{randAdap1}, \ref{poly2round} and \ref{manyRounds} in this paper. We present here the results only for $\delta \leq 1/d$ because we are generally interested in values of $\delta$ that are small.

To conclude, our new results for adaptive learning are the results in (2), (8), (10), (12) and (13).
We note also that for the deterministic case, our non-adaptive results are as good as the adaptive ones presented here.
Therefore, the deterministic adaptive algorithms don't offer us better results compared with the results of the deterministic non-adaptive algorithms. We still present these algorithms because they are useful for analyzing the randomized adaptive algorithms.


\begin{table}[H]
\small
\begin{tabular}{c||l|l|l|l|l}
  \toprule
    & \multicolumn{1}{c|}{\specialcell{Lower\\Bound}} &  \multicolumn{1}{c|}{\specialcell{Upper\\Bound}} & \multicolumn{1}{c|}{\specialcell{Time\\ $n^{O(d)}$}} &  \multicolumn{1}{c|}{\specialcell{Time\\ $poly(d^d,n)$}} & \multicolumn{1}{c}{\specialcell{Time\\ $poly(2^d,n)$}} \\
  \hline  \hline
   \multicolumn{6}{c}{{\bf Deterministic}}\\
   \hline  \hline
   Previous  & $\N{1}2^{d}\log{n}$ & \specialcell{$\underline{r=O(d\log{n})}$\\$\N{3}d2^{d}\log{n}$} & \specialcell{$\underline{r=O(d\log{n})}$\\$\N{5}d2^{d}\log{n}$} & \specialcell{$\underline{r=2}$\\$\N{7}(\log{d})d^3\cdot$ \\ \hspace{.21in}$2^d\log{n}$} & \specialcell{$\underline{r=O(d\log{n})}$\\$\N{11}2^{d+O(\log^{2}{d})}\cdot$ \\ \hspace{.5in}$\log{n}$}\\
   \hline
   Ours    &   \multicolumn{1}{c|}{---} &   \multicolumn{1}{c|}{---} &  \multicolumn{1}{c|}{---} & \specialcell{$\underline{r=2}$\\$\N{8}d^32^d\log{n}$} &   \multicolumn{1}{c}{---} \\
   \hline  \hline
   \multicolumn{6}{c}{{\bf Randomized}}\\
   \hline  \hline
   Previous &   \multicolumn{1}{c|}{---} &  \multicolumn{1}{c|}{---} &   \multicolumn{1}{c|}{---} & \specialcell{$\underline{r=2}$\\$\N{9}d\log{n}+$\\ \hspace{.17in}$d^22^d\log{\frac{1}{\delta}}$} &   \multicolumn{1}{c}{---}   \\
   \hline
   Ours &   \specialcell{$\N{2}d\log{n}+$\\ \hspace{.35in}$2^d$} &\specialcell{$\underline{r=O(d\log{d})}$\\ $\N{4}d\log{n}+$\\ \hspace{.17in}$2^d\log{\frac{1}{\delta}}$} &\specialcell{$\underline{r=O(d\log{d})}$\\ $\N{6}d\log{n}+$\\\hspace{.17in}$2^d\log{\frac{1}{\delta}}$} & \specialcell{$\underline{r=2}$\\$ \N{10}d\log{n}+$\\ \hspace{.17in}$d2^d\log{\frac{1}{\delta}}$} & \specialcell{$\underline{r=O(d\log{d})}$\\$\N{12}d\log{n}+$\\ \hspace{.23in}$2^d\log{\frac{1}{\delta}}$\\ \\ $\underline{r=2}$\\$\N{13}d\log{n}+$\\ \hspace{.23in}$d^22^d\log{\frac{1}{\delta}}$}   \\
  \bottomrule
\end{tabular}
\caption{Results for Adaptive Learning}
\label{AdaptiveTable}
\begin{tablenotes}
      \small
      \item The parameter $r$ denotes the number of rounds.
    \end{tablenotes}

\end{table}


\section{Definitions and Preliminary Results}
In this section we give some definitions and preliminary results that will be used throughout the paper.

Let $n$ be an integer. We denote $[n]=\{1,\ldots,n\}$. Consider the set of {\it assignments}, also called {\it membership queries} or {\it queries}, $\{0,1\}^n$.
A function $f$ is said to be a {\it boolean function} on $n$ variables
$x=(x_1,\ldots,x_n)$ if $f:\{0,1\}^n\to \{0,1\}$. For an assignment $a\in\{0,1\}^n$, $i\in [n]$ and $\xi\in\{0,1\}$
we denote by $a|_{x_i\gets \xi}$ the assignment $(a_1,\ldots,a_{i-1},$ $\xi,a_{i+1},\ldots,a_n)$.
We say that the variable $x_i$ is {\it relevant} in $f$
if there is an assignment $a=(a_1,\ldots,a_n)$ such that $f(a|_{x_i\gets 0})\not= f(a|_{x_i\gets 1})$.
We say that the variable $x_i$ is {\it irrelevant} in $f$ if it is not relevant in $f$.

Given a boolean function $f$ on $n$ variables $x=\{x_1,\ldots,x_n\}$ and an assignment ${a \in \{0,1\}^n}$ on the variables, we say that ${x_i}$, for $1 \leq i \leq n$, is {\it sensitive in $f$ w.r.t.} $a$ if $f({a|_{{x_i}\gets 0}}) \neq $ $f({a|_{{x_i}\gets 1}})$.

For a boolean function on $n$ variables $f$, a variable $x_i$ and a value $\xi\in\{0,1\}$ we
denote by $f|_{x_i\gets \xi}$ the boolean function on $n$ variables
$g(x_1,\ldots,x_n)=f(x_1,\ldots,x_{i-1},\xi,x_{i+1},$ $\ldots,x_n)$.
In addition, for a boolean variable $x \in \{0,1\}$ we denote $x^1 := x$ and $x^0 := \bar{x}$, where $\bar x$ is the negation of $x$.

For a class of boolean functions $C$ we say that a set of assignments $A\subseteq \{0,1\}^n$ is an {\it equivalent set for} $C$ if for every two distinct functions $f,g\in C$ there is an $a\in A$ such that $f(a) \neq g(a)$. Obviously, an equivalent set $A$ for $C$ can be used to non-adaptively learn $C$. Just ask all the the queries in $A$ and find the function $f\in C$ that is consistent with all the answers. By the definition of equivalent set this function is unique.

For integers $n$ and $d\le n$ we define the set $d$-Junta as
the set of all functions $f:\{0,1\}^n\to \{0,1\}$ with at most
$d$ relevant variables.

A set $S\subseteq \{0,1\}^n$ is called a $d$-{\it wise independent set} if for every $1\le i_1< i_2<\cdots<i_d\le n$ and every $\xi_1,\ldots,\xi_d\in \{0,1\}$, for a random uniform $x=(x_1,\ldots,x_n)\in S$ we have $\Pr [x_{i_1}=\xi_1,\ldots,x_{i_d}=\xi_d]=1/2^d.$
Alon et. al. show

\begin{lemma}\cite{ABI}\label{ABI} There is a construction of a $d$-wise independent set of size $O((2n)^{\lfloor d/2\rfloor})$ that runs in $n^{O(d)}$ time.
\end{lemma}

\subsection{Universal Sets, $d$-wise Bipartite Connected Families and Perfect Hash Families}
A set $S\subseteq \{0,1\}^n$ is called an $(n,d)$-{\it universal set} if for every $1\le i_1<i_2<\cdots<i_d\le n$
and $\sigma\in\{0,1\}^d$ there is an $s\in S$ such that $s_{i_j}=\sigma_j$ for all $j=1,\ldots,d$.
Denote by $U(n,d)$ the minimum size of an $(n,d)$-universal set. It is known,~\cite{KS72,SB}, that
\begin{eqnarray}\label{LBUS}
U(n,d)=\Omega(2^d\log n), \ \ \ U(n,d)=O(d2^d\log n)\enspace.
\end{eqnarray}
The following result is a folklore result. We give the proof for completeness.
\begin{lemma} Let be $m=O(2^d(d\log n+\log(1/\delta)))$. Let be $S=\{s^{(1)},\ldots,s^{(m)}\}$ where each $s^{(i)}$ is uniformly independently chosen from $\{0,1\}^n$.
With probability at least $1-\delta$ the set $S$ is an $(n,d)$-universal set.
\end{lemma}
\begin{proof} By the union bound and since $s^{(i)}$ are chosen uniformly independently,
\begin{eqnarray*}
\Pr[S \mbox{\ is not $(n,d)$-US}]&=&
\Pr[(\exists(i_1,\ldots,i_d)(\exists a\in\{0,1\}^d)(\forall j)\ s^{(j)}_{i_1},\ldots,s^{(j)}_{i_d}\not=a]\\
&\le& \binom{n}{d} 2^d\left(1-\frac{1}{2^d}\right)^m\le \delta\enspace.
\end{eqnarray*}
\end{proof}
This also implies the upper bound of $O(d2^d\log n)$ in (\ref{LBUS}).
Different deterministic constructions for $(n,d)$-universal sets are useful, especially for adaptive learning. The best known constructions are stated in the following.
\begin{lemma}~\cite{NSS95,B15}\label{LemmaUS} There is a
deterministic construction for an $(n,d)$-universal set of size $s$ that runs in time $T$
where
\begin{enumerate}
\item $T=poly(2^d,n)$ and $s=2^{d+O(\log^2 d)}\log n$.
\item $T=poly(d^d,n)$ and $s=O(d^32^d\log n)$.
In particular, $T=poly(n)$ for $d=O(\log n/$ $\log\log n)$.
\item $T=n^{O(d)}$ and $s=O(d2^d\log n)$.
In particular, $T=poly(n)$ for $d=O(1)$.
\end{enumerate}
\end{lemma}

Let $A\subseteq \{0,1\}^n$ be a set of assignments. For non-negative integers $d_1$, $d_2$ such that $d=d_1+d_2$ and for $i=(i_1,\ldots,i_{d_2}), j=(j_1,\ldots,j_{d_1}), k=(k_1,\ldots,k_{d_2})$, with entries that are distinct elements in $[n]$ and $z\in \{0,1\}^{d_1}$, we define a bipartite graph $B:=B(i,j,k,z,A)$ as follows. The set of vertices of the left side of the graph is $V_L(B):=\{0,1\}^{d_2}\times \{L\}$ and the set of vertices of the right side of the graph is $V_R(B):=\{0,1\}^{d_2}\times\{R\}$. For $a',a''\in \{0,1\}^{d_2}$,  $\{(a',L),(a'',R)\}$ is an edge in $B$ if and only if there is an assignment $a\in A$ such that $(a_{i_1},\ldots,a_{i_{d_2}})=a'$, $(a_{k_1},\ldots,a_{k_{d_2}})=a''$ and $(a_{j_1},\ldots,a_{j_{d_1}})=z$.

We say that $A$ is a $d$-{\it wise bipartite connected family} if for all $d_1$, $d_2$ such that $d=d_1+d_2$, and for all $i=(i_1,\ldots,i_{d_2})$, $j=(j_1,\ldots,j_{d_1})$, $k=(k_1,\ldots,k_{d_2})$ and $z\in \{0,1\}^{d_1}$ as described above, the graph $B(i,j,k,z,A)$ is connected. That is, there is a path from any vertex in the graph to any other vertex.
Obviously, any $d$-wise bipartite connected family is an $(n,d)$-universal set. Just take $d_1=0$ and $d_2=d$.

In \cite{D00}, Damaschke proves
\begin{lemma}~\cite{D00}\label{equi} A set $A$ is an equivalent set for $d$-Junta if and only if $A$ is a $d$-wise bipartite connected family.
\end{lemma}

Let $H$ be a family of functions ${h : [n] \rightarrow [q]}$, and let be ${d \leq q}$. We say that $H$ is an $(n,q,d)$-{\it perfect hash family} ($(n,q,d)$-PHF) if for every subset ${S \subseteq [n]}$ of size ${|S| = d}$ there is a hash function ${h \in H}$ such that ${|h(S)| = d}$.

In~\cite{B15} Bshouty shows
\begin{lemma}~\cite{B15}
\label{PHF}
Let $q \ge 2d^2$. There is an $(n,q,d)$-PHF of size
$$ O\left(\frac{d^2\log{n}}{\log{(q/{d^2})}}\right)$$
that can be constructed in time $O(qd^2n\log{n}/\log{(q/{d^2})})$.
\end{lemma}

In particular, from Lemma~\ref{PHF}, for $q=d^2$ and $q=d^3$ we have
\begin{corollary}
\label{PHF2}
\begin{enumerate}
\item There is a construction of an $(n,d^2,d)$-PHF of size $O(d^2 \log{n})$ that runs in $poly(n)$ time.
\item There is a construction of an $(n,d^3,d)$-PHF of size $O(d^2 \log{n} / \log{d})$ that runs in $poly(n)$ time.
\end{enumerate}
\end{corollary}

\section{Deterministic Non-adaptive Algorithms}
In this section we study deterministic non-adaptive learning of $d$-Junta.

\subsection{Lower and Upper Bound}

For completeness sake, we give a sketch of the proof of the lower bound of $\Omega(2^d\log n)$ on
the number of queries in any deterministic adaptive algorithm. This implies the same lower bound for any deterministic non-adaptive algorithm.

The idea of the lower bound is very simple. If the set of asked assignments $A$ is not an $(n,d)$-universal set and the adversary answers $0$ for all the membership queries in $A$, the learner can't learn the function. This is because, if $A$ is not an $(n,d)$-universal set, then there are $1\le i_1<i_2<\cdots< i_d\le n$ and $\xi_1,\ldots, \xi_d$ such that no assignment $a \in A$ satisfies $(a_{i_1},\ldots,a_{i_d})=(\xi_1,\ldots,\xi_d)$. Then, the learner can't distinguish between the zero function and the term $x_{i_1}^{\xi_1}\cdots x_{i_d}^{\xi_d}$. This is because $x_{i_1}^{\xi_1}\cdots x_{i_d}^{\xi_d}$ is also zero on all the assignments of $A$. Therefore, the learner must ask at least $U(n,d)$ queries which is $\Omega(2^d\log n)$ by (\ref{LBUS}).

As for the upper bound, Damaschke shows in \cite{D00} that a $d$-wise bipartite connected family is an equivalent set for $d$-Junta and therefore this family is enough for non-adaptive learning. He then shows that there exists a $d$-wise bipartite connected family of size ${O(d2^{d}\log{n}+d^22^d)}$. In this section we construct such one of size $O(d2^d\log n)$. In particular, we have

\begin{theorem}\label{Theo1} There is a deterministic non-adaptive algorithm that learns $d$-Junta
with $O(d2^d\log n)$ queries.
\end{theorem}
\begin{proof} We give an algorithmic construction of a $d$-wise bipartite connected family of size $O(d2^d\log n)$.
To prove the result we start with a definition.

For every $d_1$ and $d_2$ where $d=d_1+d_2$ and every $i=(i_1,\ldots,i_{d_2})$, $j=(j_1,\ldots,j_{d_1})$, $k=(k_1,\ldots,k_{d_2})$, with entries that are distinct elements in $[n]$, every $z\in \{0,1\}^{d_1}$ and every set of assignments $A\subseteq\{0,1\}^n$, we define the function $X_{i,j,k,z}$ at $A$ as follows. $X_{i,j,k,z}(A)=t-1$ where $t$ is the number of connected components in $B(i,j,k,z,A)$. Obviously, if $X_{i,j,k,z}(A)=0$ then $B(i,j,k,z,A)$ is connected. Consider the function
$X(A)=\sum_{i,j,k,z} X_{i,j,k,z}(A).$ The sum here is over all possible $d_1$, $d_2$, $i,j,k$ and $z$ as described above.
Notice that if $X(A)=0$ then $A$ is a $d$-wise bipartite connected family.

We construct $A$ iteratively. At the beginning $A= \varnothing$ and each $B(i,j,k,z,A)$ has $2^{d_2+1}$ connected components. Therefore, we first have
$$X(\varnothing)=\sum_{d_1+d_2=d} \binom{n} {d_2\ d_2\ d_1 \ n-d_1-2d_2} 2^{d_1}   ({2^{d_2+1}}-1) \le (2n)^{2d}\enspace.$$
We show that for every $A\subset \{0,1\}^n$ there is an assignment $a\in \{0,1\}^n$ such that
\begin{eqnarray}\label{Red}
X(A\cup\{a\})\le X(A)\left(1-\frac{1}{2^{d+1}}\right)\enspace.
\end{eqnarray}
This implies that there is a set $A'$ of size $t=2d2^{d+1}\ln (2n)$ such that
$$X(A')\le X(\varnothing)\left(1-\frac{1}{2^{d+1}}\right)^t\le (2n)^{2d}\left(1-\frac{1}{2^{d+1}}\right)^t< (2n)^{2d}e^{-2d\ln (2n)}=1\enspace.$$
Since $X(A')$ is an integer number we get $X(A')=0$, which implies that $A'$ is a $d$-wise bipartite connected family of size $t=O(d2^d\log n)$.

We now prove (\ref{Red}).
Consider some $i,j,k,z,A$. Suppose that the number of connected components in $B:=B(i,j,k,z,A)$ is $t$ and therefore $X_{i,j,k,z}(A)=t-1$.
Let $C_1,C_2,\ldots,C_t$ be the connected components of $B$ and let $s_i$ and $r_i$
be the number of vertices of the component $C_i$ in $V_L(B)$ and $V_R(B)$ respectively.
Consider a random uniform assignment $a\in \{0,1\}^n$. If $(a_{j_1},\ldots,a_{j_{d_1}})=z$ then $B(i,j,k,z,A\cup\{a\})$
is the bipartite graph $B(i,j,k,z,A)$ with an addition of a uniform random edge.
Therefore the probability that after adding $a$ to $A$ the number of connected components in $B$ reduces by $1$ is
equal to the probability that $(a_{j_1},\ldots,a_{j_{d_1}})=z$ and a uniform random edge in $B$ connects two distinct connected components. This probability is equal to
\begin{eqnarray}
\frac{1}{2^{d_1}}\frac{\sum_{i\not=j}s_ir_j}{2^{2d_2}}&=&\frac{\left(\sum_i s_i\right) \left(\sum_j r_j\right)-\sum_is_ir_i}{2^{d_1}2^{2d_2}}
\nonumber \\ &=& \frac{1}{2^{d_1}}-\frac{\sum_i{s_ir_i}}{2^{d_1}2^{2d_2}}\label{S01}\\
&\ge& \frac{1}{2^{d_1}}-\frac{2^{2d_2}-(t-1)2^{d_2-1}}{2^{d_1}2^{2d_2}}=\frac{t-1}{2^{d+1}}\enspace.\label{S02}
\end{eqnarray}
Equality (\ref{S01}) is true because $\sum_i s_i=\sum_i r_i=2^{d_2}$. The inequality (\ref{S02}) is proved later.
\ignore{
first notice that for every $j$, either $s_j>0$ or $r_j>0$ and therefore (w.l.o.g.) at least $\lceil t/2\rceil$ of the $r_j$s are not zero.
Then
\begin{eqnarray*}
2^{2d_2}&=&(\sum_i{s_i})(\sum_j{r_j}) = \sum_i{s_ir_i}+\sum_i{(s_i(\sum_{j\not=i}{r_j}))} \\
&\ge& \sum_i{s_ir_i}+(\lceil t/2\rceil -1)\sum_i{s_i} \ge  \sum_i{s_ir_i} +(t-1)2^{d_2-1}.
\end{eqnarray*}}
Therefore $$\E_{a}[X_{i,j,k,z}(A\cup \{a\})]\le (t-1)-\frac{t-1}{2^{d+1}}=X_{i,j,k,z}(A)\left(1-\frac{1}{2^{d+1}}\right)\enspace.$$
Since the expectation of a sum is the sum of the expectations, we have
$$\E_{a}[X(A\cup \{a\})]\le X(A)\left(1-\frac{1}{2^{d+1}}\right)\enspace.$$
Therefore, for every set $A$ there exists an $a$ such that
$X(A\cup\{a\})\le X(A)\left(1-1/2^{d+1}\right).$

It remains to prove (\ref{S02}). That is,
\begin{eqnarray}\label{S03}
\max_{r_i,s_i}\sum_{i=1}^t s_ir_i\le 2^{2d_2}-(t-1)2^{d_2-1}\enspace.
\end{eqnarray}
First notice that since the graph $B$ is a bipartite graph we have that $s_ir_i=0$ if and only if either $s_i=0$ and $r_i=1$ or $s_i=1$ and $r_i=0$. We first claim that the maximum value in (\ref{S03}) occurs when $r_is_i=0$ for all $i$ except for one. If, on the contrary, the maximum value occurs where $s_{i_1}r_{i_1}\not=0$ and $s_{i_2}r_{i_2}\not=0$ for some $i_1\not=i_2$ then by replacing $s_{i_1},r_{i_1}$ by $0,1$ and $s_{i_2},r_{i_2}$ by $s_{i_1}+s_{i_2}, r_{i_1}+r_{i_2}-1$ we get a larger or equal value and therefore we get a contradiction. Therefore we may assume w.l.o.g. that $s_ir_i=0$ for all $i=1,\ldots,t-1$, $r_t=2^{d_2}-t_1$ and $s_t=2^{d_2}-t_2$ for some $t_1+t_2=t-1$. Then, since $t\le 2^{d_2+1}$,
\begin{eqnarray*}
\max_{r_i,s_i}\sum_{i=1}^t s_ir_i&=&\max_{t_1+t_2=t-1}(2^{d_2}-t_1)(2^{d_2}-t_2)\\
&=&2^{2d_2}-(t-1)2^{d_2}+\max_{t_1+t_2=t-1}t_1t_2\\
&\le& 2^{2d_2}-(t-1)2^{d_2}+\frac{(t-1)^2}{4}\le 2^{2d_2}-(t-1)2^{d_2-1}\enspace.
\end{eqnarray*}
\end{proof}

\subsection{Polynomial Time Algorithms}

In this section we give three polynomial time algorithms for non-adaptive learning of $d$-Junta. The first algorithm asks $O(d2^d\log n)$ queries and runs in time $n^{O(d)}$ which is polynomial for constant $d$. This improves the query complexity $O(d^22^d\log n)$ of Damaschke in \cite{D00}.
The second algorithm asks $O(d^32^d\log n)$ queries and runs in time $poly(d^d,n)$ which is polynomial for $d=O(\log n/$ $\log\log n)$. The third algorithm asks $2^{d+O(\log^2 d)}\log n$ queries and runs in polynomial time. This improves the query complexity $2^{d+O(\log^2 d)}\log^2 n$ of Damaschke in \cite{D00}.

\vspace{\baselineskip}
We now present the first algorithm. In the next lemma we show how to construct a $d$-wise bipartite connected family of size $O(d2^d\log n)$ in time $n^{O(d)}$. We construct a $d$-wise bipartite connected family $A$ and non-adaptively ask all the queries in $A$. Then, for every $d$ variables $x_{i_1},\ldots,x_{i_d}$ we construct a set $$M_{i_1,\ldots,i_d}=\{(a_{i_1},\ldots,a_{i_d},f(a))\ |\ a\in A\}\enspace.$$ We now look for a function $g \in d$-Junta that is consistent with all the answers of the queries in $A$. If
$M_{i_1,\ldots,i_d}$ contains two elements $(a',0)$ and $(a',1)$ for some $a'\in \{0,1\}^d$ then no consistent function exists on those variables. Otherwise, there is a consistent function $g$ and since $A$ is an equivalent set, $g$ is unique and is the target function. After finding the target function, we can then find the set of relevant variables in $g$ from its truth table if needed. This set of relevant variables is a subset of $\{i_1,\ldots,i_d\}$.

\ignore{
We start with a definition: for a fixed function $f$ with $n$ variables $x=\{x_1,\dots,x_n\}$ and a fixed set of assignments $A \in \{0,1\}^n$, a subset $u \subseteq x$ is called {\it $(A,f)$-feasible} if all assignments $a \in A$ agreeing on $u$ yield the same $f(a)$.

Damaschke's proof of Lemma~\ref{equi}, in \cite{D00}, provides a non-adaptive algorithm for learning: for a function $f \in d$-Junta and a $d$-wise bipartite connected family $A$, find the intersection $F$ of all $(A,f)$-feasible sets of size at most $d$. $F$ is equal to the set of relevant variables $D$. Now, since $A$ is $(n,d)$-universal, knowing $D$ is enough for learning.
Finding $F$ can be done in $n^{O(d)}$ time, by checking all subsets of size at most $d$.}

This algorithm runs in time $n^{O(d)}$. We now show that the construction time of the $d$-wise bipartite connected family $A$ is $n^{O(d)}$.

In the previous subsection we showed in Theorem~\ref{Theo1} an algorithmic construction of a $d$-wise bipartite connected family of size $O(d2^{d}\log{n})$. We now show that this construction can be performed in $n^{O(d)}$ time.

\begin{lemma}
\label{detPoly}
There is an algorithm that runs in time $n^{O(d)}$ and constructs a $d$-wise bipartite connected family of size $O(d2^{d}\log{n})$.
\end{lemma}

\begin{proof} Let $X(A)$ and $X_{i,j,k,z}(A)$ be as in the proof of Theorem~\ref{Theo1}. $X_{i,j,k,z}(A)$ depends only on $2d_2+d_1$ entries of the assignments of $A$. Therefore, the proof of Theorem~\ref{Theo1} remains true if the new assignment $a$ is chosen from a $(2d_2+d_1)$-wise independent set $S$. By Lemma~\ref{ABI}, such set exists and is of size $n^{O(2d_2+d_1)}=n^{O(d)}$.

Since the number of iterations and the number of random variables $X_{i,j,k,z}(A)$ in the algorithm is at most $(2n)^{2d}$ and each $X_{i,j,k,z}(A)$ can be computed in time $poly(2^{d_2},|A|)=poly(2^d\log n)$, the result follows.
\end{proof}

 \noindent
 {\bf Remark.} We note here that instead of using a $(2d_2+d_1)$-wise independent set, one can use a $(1/2^{O(d)})$-almost $(2d_2+d_1)$-wise independent set of size $poly(2^d,\log n)$,~\cite{B16}. For $k=2d_2+d_1$ this is a set of assignments $S\subseteq \{0,1\}^n$ such that for every $1\le i_1< i_2<\cdots<i_k\le n$ and every $B\subseteq  \{0,1\}^n$, for a random uniform $x=(x_1,\ldots,x_n)\in S$ we have $|B|/2^{k}-1/2^{O(k)}\le \Pr [(x_{i_1},\ldots,x_{i_k})\in B]\le |B|/2^{k}+1/2^{O(k)}.$ The algorithm still needs time $n^{O(d)}$ for computing~$X(A)$.

 Another approach is the conditional probability method \cite{MR95}. It follows from the proof of Theorem~\ref{Theo1} that for $t=d2^{d+2}\ln (2n)$ i.i.d. random uniform assignments $A=\{a^{(1)},\ldots,a^{(t)}\}$ we have $\E[X(A)]<1$. We now construct the bits of $a^{(1)},\ldots,a^{(t)}$ one at a time while maintaining the property $\E[X(A)|$ {\it already fixed bits}$]<1$. At the end all the bits are fixed, say $A=A_0$, and then $X(A_0)=\E[X(A)|A_0]<1$ which implies (you can see why in the proof of Theorem~\ref{Theo1}) that $A_0$ is a $d$-wise bipartite connected family of size $t=d2^{d+2}\ln (2n)$. In this approach also the algorithm still needs time $n^{O(d)}$ for computing~$X(A)$.

Lemma~\ref{detPoly} implies
\begin{theorem}\label{Theo2} There is a deterministic non-adaptive algorithm that learns $d$-Junta
in time $n^{O(d)}$ with $O(d2^d\log n)$ queries.
\end{theorem}

For introducing the second algorithm, we start with presenting a result from~\cite{ABM}. A class of boolean functions $C$ is said to be {\it closed under variable projection} if for every
$f\in C$ and every function $\phi:[n]\to [n]$ we have $f(x_{\phi(1)},\ldots,x_{\phi(n)})\in C$.
Obviously, $d$-Junta is closed under variable projection.

For the second algorithm we apply the reduction described in \cite{ABM} to the above algorithm. We first give the reduction.

\begin{lemma}\label{DNtoDN}\cite{ABM} Let $C$ be a class of boolean functions that is closed
under variable projection.
If $C$ is non-adaptively learnable
in time $T(n)$ with $Q(n)$ membership queries, then $C$ is non-adaptively
learnable in time $$O\left(qd^2n\log n+\frac{d^2\log n}{\log (q/(d+1)^2)}(T(q)+Q(q)n)\right)$$
with $$O\left(\frac{d^2Q(q)}{\log (q/(d+1)^2)}\log n\right)$$ membership queries,
where $d$ is an upper bound on the number of relevant variables in $f\in C$ and
$q$ is any integer such that $q\ge 2(d+1)^2$.
\end{lemma}

We now prove
\begin{theorem}\label{Theo3} There is a deterministic non-adaptive algorithm that learns $d$-Junta
in time $poly(d^d,n)$ with $O(d^32^d\log n)$ queries.
\end{theorem}
\begin{proof} We apply the reduction in Lemma~\ref{DNtoDN} on the result from Theorem~\ref{Theo2} with $q=2d(d+1)^2$, $Q(n)=O(d2^{d}\log{n})$ and $T(n)=n^{O(d)}$. The time complexity is
$$O(2d(d+1)^2d^2n\log{n}+(d^2\log{n} / \log{d}) \cdot ((2d(d+1)^2)^{O(d)} +d2^d\log{(2d(d+1)^2)}n))$$
which is $poly(d^d,n)$, and the query complexity is
$$O(d^2 d2^d\log{(2d(d+1)^2)}\log{n} / \log{d})=O(d^32^d\log{n})\enspace.$$
\end{proof}

We now present the third algorithm. We first give a simple algorithm that runs in polynomial time and uses $2^{d+O(\log^{2}{d})}n\log{n}$ queries, and then we use the reduction in Lemma~\ref{DNtoDN}.

The algorithm first constructs an $(n,d)$-universal set $U$. Then, the algorithm replaces each assignment $a \in U$ by a block of ${n+1}$ assignments in the following way: it keeps the original assignment. In addition, for each index ${1 \leq i \leq n}$, it adds the assignment ${a|_{{x_i} \gets {\bar{{a_i}}}}}$ where $x_i$ is the $i$-th variable of the target function. Denote the new set of assignments by $U'$. After asking $U'$, we can find the set of relevant variables by comparing the value of the target function $f$ on the first assignment in each block with the value of $f$ on each one of the other assignments in the block. Since $U$ is a universal set, we can now find the function.

Now, we use a polynomial time construction of an $(n,d)$-universal set of size \\$2^{d+O(\log^{2}{d})}\log{n}$, as described in Lemma~\ref{LemmaUS}, and apply the reduction in Lemma~\ref{DNtoDN} to this algorithm for $q=2(d+1)^2$. We get a new non-adaptive algorithm that runs in polynomial time and asks $2^{d+O(\log^{2}{d})}\log{n}$ queries.

We now give a formal proof.
\begin{theorem}\label{Theo4} There is a deterministic non-adaptive algorithm that learns $d$-Junta
in polynomial time with $2^{d+O(\log^{2}{d})}\log n$ queries.
\end{theorem}
\begin{proof} Consider the above algorithm. Let $f$ be the target function and let $x_{i_1},\ldots,x_{i_{d'}}$, $d'\le d$ be the relevant variables. Then, we can write $f=g(x_{i_1},\ldots,x_{i_{d'}})$.
For any $1 \le j \le d'$, since $x_{i_j}$ is a relevant variable, there is an assignment $b\in \{0,1\}^{d'}$ such that $g(b)\not=g(b|_{x_{i_j}\gets \bar b_{i_j}})$. Since $U$ is an $(n,d)$-universal set, there is an assignment $a \in U$ such that $(a_{i_1},\ldots,a_{i_{d'}})=(b_{i_1},\ldots,b_{i_{d'}})$. Therefore
$$f(a|_{x_{i_j}\gets \bar a_{i_j}})=g(b|_{x_{i_j}\gets \bar b_{i_j}})\not= g(b)=f(a)\enspace.$$
This shows that the above algorithm can discover all the relevant variables after asking the set of assignments $U'$.

Since $U$ is an $(n,d)$-universal set, by looking at the entries of the relevant variables in $U$ we can find all the possible assignments for $x_{i_1},\ldots,x_{i_{d'}}$. Therefore, $g$, and consequently $f$, can be uniquely determined.

This algorithm asks $2^{d+O(\log^{2}{d})}n\log{n}$ queries and runs in polynomial time.
Finally, we use the reduction in Lemma~\ref{DNtoDN} with $q=2(d+1)^2$, $Q(n)=2^{d+O(\log^{2}{d})}n\log{n}$ and $T(n)=poly(2^d,n)$, and this completes the proof.
\end{proof}

\section{Randomized Non-adaptive Algorithms}
In this section we study randomized non-adaptive learning of $d$-Junta.

\subsection{Lower Bound}

The lower bound for deterministic algorithms does not imply the same lower bound for randomized algorithms. To prove a lower bound for randomized non-adaptive algorithms we use the minimax technique. We prove
\begin{theorem} Let $d<n/2$. Any Monte Carlo randomized non-adaptive algorithm for learning $d$-Junta must ask at least $\Omega(2^d\log n)$ membership queries.
\end{theorem}
\begin{proof} Let ${\cal A}(s,Q_f)$ be a Monte Carlo randomized non-adaptive algorithm that learns $d$-Junta with success probability at least $1/2$, where $s\in \{0,1\}^*$ is the random seed and $Q_f$ is the membership oracle to the target $f$. Since ${\cal A}(s,Q_f)$ is Monte Carlo it stops after some time $T$ and therefore
we may assume that $s\in \{0,1\}^T$.
Consider the random variable $X(s,f)\in\{0,1\}$ that is equal to $1$ if and only if the algorithm ${\cal A}(s,Q_f)$ returns~$f$. Then, $\E_s[X(s,f)]\ge 1/2$ for all $f$.

Consider the set of functions
$$F=\{f_{\xi,i}:=x_1^{\xi_1}\cdots x_{d-1}^{\xi_{d-1}}x_i\ |\ \xi_1,\ldots,\xi_{d-1}\in\{0,1\}, n\ge i\ge d\}$$
and the uniform distribution $U_F$ over $F$. Then
\begin{eqnarray}\label{c01}
\max_s\E_{U_F}[X(s,f)]\ge \E_{s}\left[\E_{U_F}[X(s,f)]\right]=\E_{U_F}\left[\E_s[X(s,f)]\right]\ge \frac{1}{2}\enspace.
\end{eqnarray}
Consider any seed $s'$ that maximizes $\E_{U_F}[X(s,f)]$. Then $\E_{U_F}[X(s',f)]\ge 1/2$. Let $A_{s'}=\{a^{(1)},\ldots,a^{(m_{s'})}\}$ be the queries asked by the algorithm when it uses the seed $s'$. Note that since the algorithm is non-adaptive, $m_{s'}$ is independent of $f$. Since the query complexity of a Monte Carlo algorithm is the worst case complexity over all $s$ and $f$, we have that $m_{s'}$ is a lower bound for the query complexity. So it is enough to show that $m_{s'}=\Omega(2^d\log n)$.

Define for every vector $\xi=(\xi_1,\ldots,\xi_{d-1})\in\{0,1\}^{d-1}$ the subset $$A_{s'}(\xi)=\{a\in A_{s'}\ |\ (a_1,\ldots,a_{d-1})=\xi\}\subseteq A_{s'}\enspace.$$ Notice that $\{A_{s'}(\xi)\}_\xi$ are disjoint sets. Suppose that at least $3/4$ fraction of $\xi$ satisfy $|A_{s'}(\xi)|\le \log (n-d+1)-3$. Then, for a random uniform $f_{\xi',i'}\in F$ (and therefore, random uniform $\xi'$), with probability at least $3/4$ we have $|A_{s'}(\xi')|\le \log (n-d+1)-3$. For any other assignment $a\in A_{s'}\backslash A_{s'}(\xi')$ we have $f_{\xi',i'}(a)=0$ so no information about $i'$ can be obtained from these assignments. If $|A_{s'}(\xi')|\le \log (n-d+1)-3$ then there are at most $(n-d+1)/8$ distinct values that any $f\in \{f_{\xi',j}\}_j$ can take on $A_{s'}(\xi')$ and therefore the probability to find $i'$ is at most $1/4$. Therefore, if at least $3/4$ fraction of $\xi$ satisfy $|A_{s'}(\xi)|\le \log (n-d+1)-3$, then the probability of success, $\E_{U_F}[X(s',f)]$, is at most $1-(3/4)^2=7/16<1/2$.
This gives a contradiction. Therefore for at least $1/4$ fraction of $\xi$ we have $|A_{s'}(\xi)|> \log (n-d+1)-3$. Then
$$m_{s'}=|A_{s'}|=\sum_{\xi\in\{0,1\}^{d-1}}|A_{s'}(\xi)|\ge \left(\frac{1}{4}2^{d-1}\right)(\log (n-d+1)-3)=\Omega(2^d\log n)\enspace.$$

\end{proof}

\subsection{Upper Bound and Polynomial Time Algorithms}
If randomization is allowed, for some cases the performance of our deterministic non-adaptive algorithms is satisfying when comparing with algorithms that take advantage of the randomization. This applies for $n^{O(d)}$ time algorithms, where the algorithm from Theorem~\ref{Theo2} gives good results. This algorithm provides also the upper bound of $O(d2^d\log{n})$ queries for randomized non-adaptive learning. In addition, for $poly(d^d,n)$ time algorithms, we can apply the algorithm from Theorem~\ref{Theo3} that asks $O(d^32^d\log{n})$ queries.

But, this is not the case for $poly(2^d,n)$ time algorithms. In this case we can improve over the deterministic result for certain values of the failure probability $\delta$. We next present a new Monte Carlo non-adaptive algorithm that runs in $poly(2^d,n,\log{(1/\delta)})$ time.

We first prove
\begin{lemma}\label{univ} Let be $1\le i_1\le i_2\le \cdots \le i_d\le n$ and $B\subseteq \{0,1\}^d$. If we randomly uniformly choose $2^d(\ln |B|+\ln(1/\delta))$ assignments $A\subseteq \{0,1\}^n$, then with probability at least $1-\delta$ we have: for every $b\in B$ there is an $a\in A$ such that $(a_{i_1},\ldots,a_{i_d})=b$.
\end{lemma}
\begin{proof}
The probability of failure is at most
$|B|(1-2^{-d})^{|A|}\le |B|e^{-|A|/2^d}=\delta.$
\end{proof}

We say that the variable $x_i$ {\it is sensitive in $f$ with respect to} an assignment $a$ if $f(a|_{x_i\gets 0})\not= f(a|_{x_i\gets 1})$. Obviously, if $x_i$ is relevant in $f$ then there is an assignment $a$ where $x_i$ is sensitive in $f$ with respect to $a$. If $x_i$ is irrelevant then $x_i$ is not sensitive in $f$ with respect to any assignment.

We now prove
\begin{lemma}\label{Sep} Let $f$ be a $d$-Junta function. Let $a$ be an assignment that $x_j$ is sensitive in $f$ with respect to, and let $x_i$ be an irrelevant variable in $f$. Let $b$ be a random assignment where each entry $b_\ell\in\{0,1\}$ is independently chosen to be $1$ with probability $1/(3d)$. Then
$$\Pr_b[f(a+b)=f(a)\mbox{\ and \ } b_i=1]\ge \frac{0.2}{d}$$
and
$$\Pr_b[f(a+b)=f(a)\mbox{\ and \ } b_j=1]\le \frac{0.1}{d}\enspace.$$
\end{lemma}
\begin{proof} If $b_\ell=0$ for all the relevant variables $x_\ell$ of $f$ then $f(a+b)=f(a)$.
Therefore $\Pr_b[f(a+b)=f(a)\mbox{\ and \ } b_i=1]$ is greater or equal to the probability that $b_\ell=0$ for all the relevant variables $x_\ell$ of $f$ and $b_i=1$. This probability is at least
$$\left (1-\frac{1}{3d}\right)^d\frac{1}{3d}\ge \frac{0.2}{d}\enspace.$$

If $x_j$ is sensitive in $f$ with respect to $a$ then the probability that $f(a+b)=f(a)$ when $b_j=1$ is less or equal to the probability that $b_\ell=1$ for some other relevant variable $x_\ell$ of $f$. Therefore
\begin{eqnarray*}
\Pr_b[f(a+b)=f(a)\mbox{\ and \ } b_j=1]&=&\Pr[b_j=1]\Pr[f(a+b)=f(a)|b_j=1]\\
&\le & \frac{1}{3d}\cdot \left(1-\left(1-\frac{1}{3d}\right)^{d-1}\right)\le \frac{0.1}{d}\enspace.
\end{eqnarray*}
\end{proof}

From Chernoff bound it follows that
\begin{lemma}\label{Cher} Let $f$ be a $d$-Junta function. Let $a$ be any assignment. There is an algorithm that asks
$O(d(\log n+\log 1/\delta))$ membership queries and with probability at least $1-\delta$ finds all the sensitive variables in $f$ with respect to $a$ (and maybe other relevant variables of $f$ as well).
\end{lemma}

Now we give the algorithm
\begin{theorem}
	\label{boundedDelta}
	 $d$-Junta is non-adaptively learnable in $poly(2^d,n,\log{1/\delta})$ time and $O((d2^d\log n+d^22^d)\log(1/\delta))$ queries, where $\delta$ is the failure probability and it satisfies $\delta \leq 1/d$.
\end{theorem}

\begin{proof}
 Consider the following algorithm. We first choose $t=O(2^d(\log d+\log(1/\delta)))$ random uniform assignments $A$. To find the relevant variables we need for each one an assignment that the variable is sensitive in $f$ with respect to it. Therefore, we need at most $d$ such assignments.
By Lemma~\ref{univ}, with probability at least $1-(\delta/3)$, for every relevant variable in $f$ there is an assignment $a$ in $A$ such that this variable is sensitive in $f$ with respect to it.
Now, by Lemma~\ref{Cher}, for each $a\in A$ there is an algorithm that asks $O(d(\log n+\log (t/\delta)))$ membership queries and with probability at least $1-(\delta/(3t))$ finds all the variables that are sensitive to it. Therefore, there is an algorithm that asks $O(dt(\log n+\log (t/\delta)))$ membership queries and with probability at least $1-(\delta/3)$ finds every variable that is sensitive in $f$ with respect to some assignment in $A$. This gives all the relevant variables of $f$. Now again by Lemma~\ref{univ}, with another $O(2^d(d+\log(1/\delta)))$ assignments we can find, with probability at least $1-(\delta/3)$, the value of the function in all the possible assignments for the relevant variables. At the end, we can output the set of relevant variables and a representation of the target function as a truth table (with the relevant variables as entries).

This algorithm runs in time $poly(2^d,n,\log{(1/\delta)})$ and asks

\begin{eqnarray}\label{initial}
\nonumber &O&(2^d(\log d+\log(1/\delta))\cdot d(\log n+\log(1/\delta)+d)+2^d(d+\log(1/\delta)))\\
=&O&(d2^d(\log d+\log(1/\delta))\log n+d2^d(d+\log(1/\delta))(\log d+\log(1/\delta)))
\end{eqnarray}

 membership queries. For $\delta=1/d$ we have the complexity
$$O((\log d)d2^d\log n+(\log d)d^22^d)\enspace.$$

For $m$ repetitions of the algorithm with at least one success, we can find the correct target function by the following. First, for each output function, verify that each claimed relevant variable
is indeed relevant with respect to the function's truth table. Discard functions that don't satisfy this. Second, take the function with the maximum number of
relevant variables out of the remaining functions. This is the correct target function. This can be done in $poly(n,2^d,m)$ time.

Therefore, for $\delta \leq 1/d$, we can repeat the above algorithm $O(\log(1/\delta)/\log d)$ times (non-adaptively) to get a success probability of at least $1-\delta$ and a query complexity of
$$O((d2^d\log n+d^22^d)\log(1/\delta))\enspace.$$

\end{proof}

Notice that if $\delta=1/poly(n)$ then the query complexity becomes quadratic in $\log n$.
In the next subsection we solve this problem by giving a reduction that changes the query complexity to
$$O((\log{d})d^22^d(\log n+\log(1/\delta)))\enspace.$$


\subsection{A Reduction for Randomized Non-adaptive Algorithms}

In this subsection we give a reduction for randomized non-adaptive algorithms. Using this reduction we prove
\begin{theorem}
\label{randRed}
There is a Monte Carlo randomized non-adaptive algorithm for learning $d$-Junta in $poly(2^d,n,\log{(1/\delta)})$ time that asks
$O((\log{d})d^22^d(\log n+\log(1/\delta)))$ membership queries.
\end{theorem}

This result improves over the result from the previous subsection for values of $\delta$ small enough.

We start with some definitions.
Let $C$ be a class of functions and let $H$ be a class of representations of functions. We say that $C$ is a {\it subclass} of $H$ if for every function $f$ in $C$ there is at least one representation $h$ in $H$. We say that $C$ is {\it non-adaptively learnable from} $H$ if there is a non-adaptive learning algorithm that for every function $f \in C$ outputs a function $h\in H$ that is equivalent to the target function $f$. We say that $C$ is {\it closed under distinct variable projection} if for any function $f\in C$ and any permutation $\phi:[n]\to [n]$ we have $f(x_{\phi(1)},\ldots,x_{\phi(n)})\in C$.

We now prove

\begin{theorem} Let $C$ be a class of boolean functions that is closed
under distinct variable projection. Let $d$ be an upper bound on the number of relevant variables in any $f\in C$. Let $H$ be a class of representations of boolean functions.
Let $h_1,h_2$ be functions in $H$ with at most $d$ relevant variables, where these relevant variables are known.
Suppose there is a deterministic algorithm ${\cal B}(d,n)$ that for such input $h_1$ and $h_2$ decides whether the two functions are
equivalent in time $E(d,n)$.

Let ${\cal A}(d,n,\delta)$ be a Monte Carlo non-adaptive algorithm that with failure probability at most $\delta$, learns $C$ from $H$ and finds the set of relevant variables of the input in time $T(d,n,\delta)$ with $Q(d,n,\delta)$ membership queries. Then, $C$ is Monte Carlo non-adaptively
learnable from $H$ in time $$O((T(d,q,1/8)n+E(d,q)(\log n+\log(1/\delta)))(\log n+\log(1/\delta)))$$
with $$O(Q(d,q,1/8)\left({\log n+\log(1/\delta)}\right))$$ membership queries, where $\delta$ is the failure probability and $q$ is any integer such that $q \geq 8d^2$.
\end{theorem}

\begin{figure}[h!]
  \begin{center}
  \fbox{\fbox{\begin{minipage}{28em}
  \begin{tabbing}
  xxxx\=xxxx\=xxxx\=xxxx\= \kill
 {\underline{\bf Reduction for Randomized Non-adaptive Algorithms}}\\
${\cal A}(n,d,\delta)$ is a non-adaptive learning algorithm for $C$ from $H$.\\
\>  ${\cal A}$ also outputs the set of relevant variables of its input.\\
\>  ${\cal A}$ runs in time $T(d,n,\delta)$.\\
1) $P\gets$ Choose $O(\log(n/\delta))$ random hash functions $h:[n]\to [q]$.\\
2) For all $h\in P$ in parallel\\
\>  Run ${\cal A}(q,d,1/8)$ to learn $f_h:=f(x_{h(1)},\ldots,x_{h(n)})$\\
\hspace{3em} and find its relevant variables.\\
\>  Stop after $T(d,q,1/8)$ time.\\
\>  For processes that do not stop output the function $0$,\\
\hspace{3em} and an empty set of relevant variables.\\
\>  Let $f_h'\in H$ be the output of ${\cal A}$ on $f_h$.\\
\>  Let $V_h$ be the set of relevant variables that ${\cal A}$ outputs on $f_h$.\\
3) $W_h=\{x_i\ |\ x_{h(i)}\in V_h\}$.\\
\hspace{0.9em}  $W\gets$ Variables appearing in more than 1/2 of the $\{W_h\}_{h \in P}$.\\
4) $T \gets$ All $h \in H$ such that for each $v_i \in V_h$ there is\\
\hspace{3.1em} exactly one variable $x_j \in W$ for which $h(j)=i$.\\
5) For each $h\in T$\\
\> $g_h\gets$ Replace each relevant variable $v_i$ in $f_{h}'$\\
\hspace{4.2em} by $x_j\in W$ where $h(j)=i$.\\
6) Output $\pop(\{g_h\}_{h \in T})$.
  \end{tabbing}
  \end{minipage}}}
  \end{center}
\caption{Reduction for Randomized Non-adaptive Algorithms}
\label{AlgIR}
\end{figure}

\begin{proof}
Let ${\cal A}(d,n,\delta)$ and  ${\cal B}(d,n)$ be as stated in the theorem above.
Let $f \in C$ be the target function and suppose the relevant
variables in $f$ are $x_{i_1},\ldots,x_{i_{d'}}$ where $d'\le d$.
Let be $I=\{i_1,\ldots,i_{d'}\}$.

We first choose $O(\log(n/\delta))$ random hash
functions $h:[n]\to [q]$ and put them in $P$.
For each hash function $h\in P$ we define the following events.
The event $A_h$ is true if $h(I):=\{h(i_1),\ldots,h(i_{d'})\}$ are not distinct.
The event $B_{h,j}$, $j\in [n]\backslash I$, is true if $h(j)\in h(I)$.
For any $h \in P$, the probability that $A_h$ is true is at most
$$\sum_{i=1}^{d-1} \frac{i}{q} = \frac{d(d-1)}{2q} \leq \frac{1}{16}\enspace.$$
For any $h \in P$ and $j\in [n]\backslash I$, the probability that $B_{h,j}$ is true is at most $d/q \leq 1/8$.

By Chernoff bound, with failure probability at most $\delta/(3n)$, we have that at least 7/8 of $\{A_h\}_{h \in P}$ are false.
Therefore, with failure probability at most $\delta/(3n)$, at least 7/8 of $\{f_h:=f(x_{h(1)},\ldots,x_{h(n)})\}_{h \in P}$ are still in $C$. This is true because $C$ is closed under distinct variable projection.

Let $V_h$ be the set of relevant variables that ${\cal A}$ outputs when running with the input $f_h$. Since hashing can not raise the number of relevant variables in the function, we can now for each $h \in P$ run in parallel ${\cal A}(d,q,1/8)$ to learn $f_h$ and find $V_h$.
Let $S \subseteq P$ denote the set of $h \in P$ where the corresponding processes finish after $T(d,q,1/8)$ time. For each $h \in S$, denote the output of ${\cal A}$ on $f_h$ by $f'_h$. With failure probability at most $\delta/(3n)$, it holds that $|S| \geq 7/8|P|$. This is true because ${\cal A}(d,q,1/8)$ stops after $T(d,q,1/8)$ time for each function in $C$, and with failure probability $\delta/(3n)$ at least 7/8 of $\{f_h\}_{h \in P}$ are in $C$. For any other $h \in P$, we stop its process and set $f'_h=0$, $V_h=\varnothing$. Applying Chernoff bound on $\{f'_h\}_{h \in S}$ yields that for at least 6/7 of them ${\cal A}$ succeeds, with failure probability at most $\delta/(3n)$. Therefore, with failure probability at most $2\delta/(3n) \leq 2\delta/3$ we have that for at least $7/8 \times 6/7= 3/4$ of ${h \in P}$ it holds that $h(I)$ are distinct, $f'_h = f_h$ and the set of relevant variables $V_h$ is correct.

For each $h \in P$ define the set $W_h=\{ x_i\ |\ x_{h(i)}\in V_h\}$ and let $W$ be the set of all the variables appearing in more than 1/2 of the $\{W_h\}_{h \in P}$.
We now find the probability that a relevant variable of $f$ is in $W_h$
and compare it with the probability that an irrelevant variable of $f$ is in $W_h$.
For $1 \leq k \leq d'$, the probability that $x_{i_k}\not\in W_h$ is at most the probability that $A_h$ is true or that
${\cal A}$ fails. This probability is at most
$1/16+1/8=3/16.$ Therefore the probability that a relevant variable
is in $W_h$ is at least 13/16.
Therefore, by Chernoff bound, the probability that $x_{i_k}\not\in W$ is at most $\delta/(3n)$.
The probability that an irrelevant variable $x_j$, $j\not\in I$, is in $W_h$ is at most
the probability that $A_h$ or $B_{h,j}$ is true or that ${\cal A}$ fails. This is bounded by $1/16+1/8+1/8=5/16$.
Therefore, by Chernoff bound, the probability that $x_{j} \in W$ for $j\not\in I$ is at most $\delta/(3n)$.

Therefore, when running the algorithm for $O(\log(n/\delta))$ random hash functions, $W$ contains all the relevant variables of the target function $f$ and only them, with probability at least $1-\delta/3$.

Let $T$ be the set of all $h \in H$ such that for each $v_i \in V_h$ there is exactly one variable $x_j \in W$ for which $h(j)=i$.
For each $h \in T$, let be $V_h = \{v_{i_1},\ldots,v_{i_{l_h}}\}$, where $l_h \leq d'$, and $f'_h:=f'_h(v_{i_1},\ldots,v_{i_{l_h}})$. Define $g_h:=f'_h(x_{j_1},\ldots,x_{j_{l_h}})$ where $h(j_k)=i_k$.
Now, with failure probability at most $2\delta/3+\delta/3=\delta$, it holds that at least 3/4 of $\{g_h\}_{h \in T}$ are identical to the target function $f$.
Therefore, at the end we use ${\cal B}(d,q)$ to find $\pop(\{g_h\}_{h \in T})$ and output it.

\end{proof}

Let $C$ be $d$-Junta and let $H$ be a class of representations of boolean functions as truth tables with the relevant variables of the function as entries. Let ${\cal A}$ be the algorithm described in the previous subsection with the query complexity in (\ref{initial}). This algorithm learns $d$-Junta from $H$ and outputs the set of relevant variables. Let ${\cal B}$ be the simple algorithm that given two functions in $H$ with at most $d$ relevant variables where these relevant variables are known, decides if they are identical by comparing the sets of relevant variables and the values of the two functions on the $O(2^d)$ entries in the tables. This algorithm runs in $poly(2^d,n)$ time. We can now apply the reduction for $q=8d^2$ and get the result from Theorem~\ref{randRed}.


\section{Deterministic Adaptive Algorithms}
In this section we study deterministic adaptive learning of $d$-Junta.

\subsection{Lower and Upper Bound}
For deterministic learning, the lower bound is $\Omega(2^d\log{n})$. Its proof is discussed earlier in the section for non-adaptive deterministic learning.
We now address the upper bound. In~\cite{D00} Damaschke gives an adaptive algorithm that learns $d$-Junta with $U(n,d)+O(d\log{n})$ queries, where $U(n,d)$ denotes the minimum size of an $(n,d)$-universal set. His algorithm first fixes an $(n,d)$-universal set $A$ and then asks all the queries $a \in A$. The algorithm can then find all the relevant variables in $d'$ search phases, where $d' \leq d$ is the number of relevant variables. Damaschke shows that at the beginning of each phase of the algorithm, the algorithm can either determine the non-existence of further relevant variables or discover a new one in $\log{n}$ rounds, asking one query each round. Since $A$ is a universal set, after the algorithm has discovered all relevant variables, it can now find the function. This, along with the upper bound on the size of a universal set from~(\ref{LBUS}), gives the upper bound of $O(d2^d\log{n})$ queries. The number of rounds is $O(d\log{n})$. We note here that our upper bound for deterministic non-adaptive learning achieves the same query complexity of $O(d2^d\log{n})$.

\subsection{Polynomial Time Algorithms}
To address polynomial time algorithms, we first consider the algorithm from the previous subsection. If we analyze the running time of this algorithm, we notice that apart from the construction time of the universal set, the algorithm runs in polynomial time. Therefore, for a construction of an $(n,d)$-universal set of size $S$ that runs in time $T$, this algorithm asks $S+O(d\log{n})$ queries and runs in time $T+poly(2^d,n)$. This, along with Lemma~\ref{LemmaUS}, gives the following results: there is an algorithm that runs in $n^{O(d)}$ time and asks $O(d2^d\log{n})$ queries, and there is an algorithm that runs in $poly(2^d,n)$ time and asks $2^{d+O(\log^2 d)}\log n$ queries. Both algorithms run in $O(d\log{n})$ rounds.

We note here that our deterministic non-adaptive algorithms for $n^{O(d)}$ time and $poly(2^d,n)$ time achieve the same query complexity of the above algorithms, yet run in only one round. Therefore, they might be preferable over the adaptive ones. These algorithms are described earlier in the deterministic non-adaptive section.

For $poly(d^d,n)$ time algorithms, one option is to follow the above algorithm and use a construction of an $(n,d)$-universal set that runs in $poly(d^d,n)$ time, as described in Lemma~\ref{LemmaUS}. But, we present another algorithm that will be useful in the construction of randomized algorithms in the next section. In addition, the algorithm we present uses only two rounds for learning. The algorithm is based on another one presented by Damaschke in~\cite{D98}. Damaschke shows that functions from $d$-Junta can be learned in $poly(d^d,n)$ time by a two round algorithm that asks $O((\log{d})d^32^{d}\log{n})$ queries, \cite{D98}. For his algorithm, he uses an $(n,d^2,d)$-PHF. Available $(n,d^2,d)$-PHF constructions are of size $O(d^2\log{n})$, see Corollary~\ref{PHF2}.
We can improve this result if we use a $poly(n)$ time construction of an $(n,d^3,d)$-PHF of size $O(d^2 \log{n}/ \log{d})$, as stated in Corollary~\ref{PHF2}. Furthermore, as part of his algorithm, Damaschke applies an earlier non-adaptive algorithm of his. We can replace it by a more efficient non-adaptive algorithm of ours, which we described earlier in the section for deterministic non-adaptive learning.

To put it in formal terms, we start with a folklore result (see for example the first observation by Blum in \cite{B03}): let $f$ be a function in $d$-Junta with a set of variables $x=\{x_1,\ldots,x_n\}$. Let $a$ and $a'$ be two assignments such that $f(a) \neq f(a')$. Let $y \subseteq x$ be the set of variables that the values of $a$ and $a'$ differ on. Then
\begin{lemma}
\label{binarySearch}
\begin{enumerate}
	\item
	We can in polynomial time and $O(\log{|y|})$ rounds find a relevant variable in $y$ asking $O(\log{|y|})$ queries. \label{prt:1}
	\item
	If, furthermore, we are guaranteed that there exists exactly one relevant variable in $y$, we can find it in only one round. The time complexity remains polynomial and the query complexity remains $O(\log{|y|})$. \label{prt:2}
\end{enumerate}
\end{lemma}

We give the proof for completeness.

\begin{proof}
\begin{enumerate}
	For proving the first part 	
	consider the following full adaptive algorithm that performs a simple binary search. Start with the given $a$, $a'$ and $y$. Let $z \subseteq y$ be a subset of $y$ that contains half of the variables in $y$. At each new round, construct a new assignment $a''$ that is identical to $a$ and $a'$ on the variables $x-y$, is equal to $a$ on $z$ and is equal to $a'$ on $y-z$. If $f(a'')\neq f(a')$, then we know there is a relevant variable in $z$. In this case update $y \gets z$ and $a \gets a''$. Otherwise we have $f(a'')\neq f(a)$ and we can update $y \gets y-z$ and $a' \gets a''$. We start a new round as long as $|y| > 1$. When $|y|=1$, we have found a relevant variable.

	The number of rounds and query complexity follow immediately if we notice that at each new round the size of $y$ reduces by a factor of 2. The time complexity is immediate.

	For proving the second part
	start by constructing a set $A$ of $\lceil \log{|y|} \rceil +1$ queries as follows.
	Let all the queries in $A$ be identical to $a$ and $a'$ on the variables $x-y$.
	In order to describe what values these queries take on $y$, we refer to the $|A| \times n$ order matrix $M$ which the queries in $A$ form its rows.
	Consider the $|y|$ columns in $M$ that correspond to the set of variables $y$. 
	Let each one of these columns take a different $\lceil \log{|y|} \rceil$ dimensional vector in the first $\lceil \log{|y|} \rceil$ rows of $M$. 
	There remains now only the last row. Set it to the value $1$ on every column that corresponds to $y$. Now, ask $A$ non-adaptively. Let $t$, $|t|=|A|$, be the vector of answers. Let $x' \in y$ be the relevant variable in $y$. Then, $t$ must either be equal to the column in $M$ corresponding to $x'$ or be equal to its negation. 
	The way $M$ is constructed guarantees there is no other column corresponding to a variable in $y$ that satisfies this. Therefore, $x'$ is successfully detected.
	
	The time and query complexities follow immediately.
\end{enumerate}
	
\end{proof}

Now, we present the algorithm.
\begin{theorem}
\label{detAdap}
There is a deterministic adaptive two round algorithm that learns $d$-Junta in time $poly(d^d,n)$ with $O(d^32^d\log{n})$ queries.
\end{theorem}

\begin{proof}
Let $f \in d$-Junta be the target function. First, construct an $(n,d^3,d)$-PHF $H$ of size $O(d^2 \log{n}/ \log{d})$ in $poly(n)$ time. For each $h \in H$, we have a new function $f_h: \{0,1\}^{d^{3}} \to \{0,1\}$ that is defined by $f_h(v) := f(x_{h(1)},\ldots,x_{h(n)})$. For a variable in $f_h$ to be a relevant variable, a necessary condition is that at least one relevant variable of $f$ is mapped to it. Therefore, we now have $O(d^2 \log{n}/ \log{d})$ new boolean functions, each of them in $d^3$ variables, of which at most $d$ are relevant. For each $f_h$, we can learn it with our deterministic non-adaptive algorithm from Theorem~\ref{Theo2}. This algorithm runs in $n^{O(d)}$ time and asks $O(d2^d\log{n})$ queries for a function in $n$ variables of which at most $d$ are relevant. We apply this algorithm simultaneously to all $\{f_h\}_{h \in H}$. Therefore, learning all $\{f_h\}_{h \in H}$ is done non-adaptively in $poly(d^d,n)$ time asking
$O((d^2 \log{n}/ \log{d}) \cdot d2^d\log{(d^{3})}) = O(d^32^d\log{n})$
queries. This is the first round of the algorithm.

Let $h' \in H$ be the mapping that succeeded to separate the relevant variables of $f$. In the second round we find the function $f_{h'}$. This is done by simply taking a function $f_{h'}$ with the highest number of relevant variables. This is true because for a variable in some $f_h$ to be a relevant variable, a necessary condition is that at least one relevant variable of $f$ is mapped to it. Therefore in this case, each relevant variable of $f_{h'}$ has exactly one relevant variable of $f$ mapped to it. Remember that the algorithm from Theorem~\ref{Theo2} gives us the set of relevant variables of the target function. Now that we know $f_{h'}$ and its relevant variables, for each relevant variable $v_i$ of $f_{h'}$, we also know an assignment $a := (a_1,\ldots,a_{d^{3}})$ such that $f_{h'}(a|_{v_i \gets 0}) \neq f_{h'}(a|_{v_i \gets 1})$. Therefore, for each relevant variable in $f_{h'}$ we can apply the algorithm from Lemma~\ref{binarySearch} part \ref{prt:2} on the variables of $f$ that are mapped to it, and find the original relevant variable. This is done non-adaptively in polynomial time using $O(\log{n})$ queries. We apply this algorithm simultaneously to all the relevant variables of $f_{h'}$. Now that we have learned $f_{h'}$ and know all the relevant variables of $f$, we have the target function $f$. This second round runs then in polynomial time and asks $O(d\log{n})$ queries.

Therefore, this algorithm learns functions from $d$-Junta in two rounds in $poly(d^d,n)$ time with a query complexity of
$O(d^32^d\log{n})+O(d\log{n}) =  O(d^32^d\log{n}).$
\end{proof}

If we compare the above algorithm with the corresponding deterministic non-adaptive algorithm that runs in $poly(d^d,n)$ time, we notice that they have the same query complexity. Even so, as mentioned before, the above algorithm is going to serve us in the next section.

\section{Randomized Adaptive Algorithms}
When we allow both randomization and adaptiveness, new bounds and algorithms arise. We discuss the new results in the following subsections.

\subsection{Lower Bound}
For randomized adaptive algorithms, Abasi et. al.~\cite{ABM} showed that any Monte Carlo randomized algorithm for learning a class of functions $C$ must ask at least $\Omega(\log |C|)$ queries. This implies
\begin{theorem} Any Monte Carlo randomized adaptive algorithm for learning $d$-Junta must ask at least $\Omega(d\log{n}+2^d)$ membership queries.
\end{theorem}
\begin{proof} Let $F_{=k}$ be the class of boolean functions in $k$ given variables $y_1,\ldots,y_k$, where each function depends on all the $k$ variables. Let $F_k$ be the class of all boolean functions in $k$ given variables $y_1,\ldots,y_k$.
Obviously, $|F_k|=2^{2^k}$. By Bonferroni inequality we have
$$|F_{=d}|\ge |F_d|-d|F_{d-1}|=2^{2^d}-d2^{2^{d-1}}\ge 2^{2^{d-1}}.$$
Therefore $$|\mbox{d-Junta}|\ge \binom{n} {d}|F_{=d}| \ge \binom{n} {d} 2^{2^{d-1}}$$ and
$\Omega(\log |\mbox{d-Junta}|)=\Omega(d\log (n/d)+2^d)=\Omega(d\log{n}+2^d)$.
\end{proof}

\subsection{Polynomial Time Algorithms}
In this subsection, for polynomial time complexity, we give three new Monte Carlo adaptive algorithms for learning $d$-Junta. We start by presenting an algorithm based on the algorithm from Theorem~\ref{detAdap}. A similar algorithm that uses $O(d\log{n}+d^22^d\log{(1/\delta)})$ queries was presented by Damaschke in~\cite{D98}. This query complexity is the result of a modification we apply to the analysis of Damaschke's algorithm, in order to make it applicable for a general failure probability $\delta$. Now we present our result.
\begin{theorem}
\label{randAdap1}
There is a two round Monte Carlo algorithm that learns $d$-Junta with $O(d\log{n}+d2^d\log{(1/\delta)})$ queries and runs in $poly(d^d,n,\log{(1/\delta)})$ time, where $\delta$ is the failure probability and $\delta \leq 1/d$. For $\delta > 1/d$, the query complexity of the algorithm is  $O(d\log{n}+(\log{d})d2^d)$.
\end{theorem}

\begin{proof}
Let ${\cal A}$ be the two round deterministic algorithm from Theorem~\ref{detAdap}. Our randomized algorithm follows a similar pattern to the pattern of $\cal A$. The difference is that it does not use a deterministic PHF. Instead, in the first round, it uses $O( \log{(1/\delta)}/\log{d})$ random partitions of the variables into $d^3$ bins for $\delta \leq 1/d$, and simply one partition for $\delta > 1/d$. Let $x=\{x_1,\ldots,x_n\}$ be the set of variables of the target function $f$. Let $D$ be the set of relevant variables. Let $P$ be the set of random partitions. For each $p \in P$, let be $f_p:=f(x_{p(1)},\ldots,x_{p(n)})$. Learn the set of functions $\{f_p\}_{p \in P}$ simultaneously applying the deterministic non-adaptive algorithm from Theorem~\ref{Theo2}. This is the first round of the algorithm, it runs in time $poly(d^d,n,\log{(1/\delta)})$. For $\delta \leq 1/d$ it asks $O(d2^d\log{(1/\delta)})$ queries, and for $\delta > 1/d$ it asks $O((\log{d})d 2^d)$ queries.

In the second round, our algorithm takes a function $f_{p'}$ with the highest number of relevant variables and uses it to learn the target function $f$. To this purpose, it applies exactly the same procedure that $\cal A$ uses in the second round. Therefore, the second round runs in polynomial time and asks $O(d\log{n})$ queries. In total, the algorithm runs in $poly(d^d,n,\log{(1/\delta)})$ time and asks $O(d\log{n}+d2^d\log{(1/\delta)})$ queries for $\delta \leq 1/d$, and $O(d\log{n}+(\log{d})d2^d)$ queries for $\delta > 1/d$.

To compute the failure probability we notice that if at least one of the partitions in $P$ succeeds to separate $D$ in the first round, then the algorithm outputs a correct answer. This is true since in this case the algorithm chooses for the second round a function $f_{p'}$ with a successful partition $p'$. Therefore, the failure probability is bounded by the probability that none of the partitions in $P$ succeeds. We first compute the failure probability of a single random partition when each element of $x$ is thrown independently and equiprobably into one of the ${d^3}$ bins. By the union bound, since $|D|\leq d$, the failure probability is bounded by
$$ \binom{d}{2} \frac{1}{d^3} \leq \frac{d^2}{d^3} = \frac{1}{d}       \enspace. $$
Therefore, one partition is enough for $\delta > 1/d$.
In addition, if $\delta \leq 1/d$, for $O( \log{(1/\delta)}/\log{d})$ random partitions, the probability that all of them fail is at most $\left(1/d\right)^{O( \log{(1/\delta)}/\log{d})}=\delta$.
\end{proof}

We were able to improve Damaschke's result by, first, using a partition with a different number of bins, and second, applying our non-adaptive algorithm in the first round.

Therefore, for $poly(d^d,n,\log{1/\delta})$ time randomized algorithms, we have two algorithms. The above algorithm and the deterministic two round algorithm from Theorem~\ref{detAdap} that asks $O(d^32^d\log{n})$ queries.
For best performance, we can choose the algorithm with minimum query complexity based on the requested value of $\delta$.
\ignore{
 We notice that for $\delta=1/(n^{\omega(d^2)})$ the deterministic algorithm has a better query complexity. For $\delta=1/(n^{o(d^2)})$ the opposite is true, and for $\delta=1/(n^{\theta(d^2)})$ the two results have the same complexity.
}


\vspace{\baselineskip}
Now, consider the following two round algorithm that runs in polynomial time. First we describe the algorithm for $\delta \leq 1/d$. In the first round this algorithm applies the same learning process as the above algorithm from Theorem~\ref{randAdap1} except for one difference. For learning the set of functions $\{f_p\}$, this algorithm applies the randomized polynomial time algorithm from Theorem~\ref{boundedDelta}. Then, the algorithm takes $\delta = 1/d$ and repeats (non-adaptively) the first round $O(\log{(1/\delta)}/\log{d})$ times. Therefore, the query complexity of the first round is
\begin{dmath*}
O\left(\left(d2^d\log{(d^3)}+d^22^d\right)\log{(1/(1/d))}\right) \cdot O\left(\frac{\log{(1/(1/d))}}{\log{d}} \right) 
 \cdot  O\left(\frac{\log{(1/\delta)}}{\log{d}}\right)= O\left(d^2 2^d \log{\frac{1}{\delta}} \right) 
\end{dmath*}
and the failure probability is $\delta$.

In the second round, this algorithm performs the same procedure as the above one from Theorem~\ref{randAdap1}. The second round therefore runs in polynomial time and asks $O(d\log{n})$ queries.

 Now, for $\delta > 1/d$, we can still apply a similar algorithm with the following changes. First, one partition of the variables into $d^3$ bins is enough.
Second, for learning the new function with $d^3$ variables in the first round, apply the algorithm with the query complexity from~(\ref{initial}). This is the initial non-adaptive polynomial time algorithm. This gives us a polynomial time algorithm that runs in two rounds. Analyzing the query complexity of this algorithm gives the result $O(d\log{n}+(\log{d})d^22^d)$.

To conclude
\begin{theorem}
	\label{poly2round}
	There is a two round Monte Carlo algorithm that learns $d$-Junta in $poly(2^d,n,\log{(1/\delta)})$ time. The algorithm asks $O(d\log{n}+d^22^d\log{(1/\delta)})$ queries for a failure probability $\delta \leq 1/d$, and asks $O(d\log{n}+(\log{d})d^22^d)$ queries for $\delta > 1/d$.
\end{theorem}


\vspace{\baselineskip}
Finally, we present a polynomial time algorithm for randomized adaptive learning that achieves a query complexity that misses the lower bound of $O(d\log{n}+2^d)$ by a factor not larger than $\log{1/\delta}$.
\begin{theorem}
	\label{manyRounds}
	Functions from $d$-Junta can be learned in $poly(2^d,n,\log{(1/\delta)})$ time by a Monte Carlo algorithm that asks $O(d\log{n}+2^d\log{(1/\delta)})$ queries, where the failure probability satisfies $\delta \leq 1/d$. For $\delta > 1/d$, the query complexity is $O(d\log{n}+(\log{d})2^d)$. The number of rounds is $O(d\log{d})$.
\end{theorem}

\begin{proof}
	We first present the algorithm for $\delta \leq 1/d$. At the end, we specify what changes should be made for the case of $\delta > 1/d$.
	
	Let $f$ be a function in $d$-Junta with a set of variables $x=\{x_1,\ldots,x_n\}$. Let the set of relevant variables be $D=\{x_{i_1},\ldots,x_{i_{d'}}\}$ where $d' \leq d$. Consider the following algorithm that starts by choosing $O\left((\log{1/\delta})/(\log{d})\right)$ random partitions $P$ of the variables into $d^3$ bins. With probability at least $1-\delta/2$, one of the partitions, $p' \in P$, succeeds to separate $D$.  Let be $f_p(y):=f(x_{p(1)},\ldots,x_{p(n)})$ for each $p \in P$. Let $D_{p'}=\{y_{t_1},\ldots,y_{t_{d'}}\}$ be the relevant bins of $f_{p'}$. For each $p \in P$ the algorithm asks a set of $2^d(\ln(2d)+\ln(2/\delta))$ random uniform assignments $A \in \{0,1\}^{d^3}$. From Lemma~\ref{univ}, with probability at least $1-\delta/2$ we have: for every relevant bin $y_{t_j} \in D_{p'}$, $1 \leq j \leq d'$, there are two assignments $a,a' \in A$ such that $a_{t_k}=a'_{t_k}$ for all the relevant bins where $k \neq j$, $a_{t_j}=1-a'_{t_j}$ and $f_{p'}(a)=1-f_{p'}(a')$. That is, $y_{t_j}$ is sensitive in $f_{p'}$ w.r.t. $a$ and w.r.t. $a'$, where $a$ and $a'$ agree on all relevant bins except for $y_{t_j}$. In other words, $a$ and $a'$ are a proof that $y_{t_j}$ is relevant, and they will serve to make sure that $y_{t_j}$ is discovered as such.
	
	After asking the set of assignments $A$ for each $p \in P$, the algorithm consists of three stages.
	In the first stage the algorithm performs $d$ or less phases of the following. Let $I_p$ be the set of discovered relevant bins of $f_p$ after phase $i$. At the beginning $I_p$ is an empty set. At phase $i+1$, the algorithm takes two assignments $b,b' \in A$ such that $f_p(b)=1-f_p(b')$ and $b,b'$ agree on all the bins in $I_p$. Then it finds a new relevant bin in $f_p$ and adds it to $I_p$. This is done using the procedure in Lemma~\ref{binarySearch} part \ref{prt:1}. This procedure is a polynomial time binary search that runs in $O(\log{(d^3)})=O(\log{d})$ rounds and asks $O(\log{(d^3)})=O(\log{d})$ queries. If there are no such $b$ and $b'$ for some $p \in P$, then the algorithm finishes the first stage for this $p$. Notice that as long as there exists a not yet discovered relevant bin in $f_{p'}$, the algorithm finds such $b,b'$ and discovers a new one. This is true with success probability $1-\delta/2$. Therefore, with success probability $1-\delta$, at the end of the first stage we have a successful partition $p'$ with $I_{p'}=D_{p'}$. The successful partition $p'$ is discovered by taking a partition $p$ with maximum number of relevant bins.
	
	This first stage is repeated non-adaptively $O\left((\log{1/\delta})/(\log{d})\right)$ times after taking $\delta = 1/d$. This optimizes the query complexity while maintaining a failure probability of $\delta$.
	
	Therefore, the number of rounds in the first stage is $O(d\log{d})$. The query complexity is
	\begin{dmath*}
	(O(2^d(\ln(2d)+\ln(2/{\delta'}))+d\log{d}) \cdot  O((\log{1/{\delta'}})/(\log{d}))) \;\vrule\;_{{\delta'} \gets 1/d} \cdot 
	 O\left((\log{1/\delta})/(\log{d})\right)	= O\left(2^d\log{1/\delta}\right) 
	\end{dmath*}

    and it runs in polynomial time.
	
	The second stage is simply to learn $f_{p'}$ in one round deterministically by asking $2^{d'}$ queries that take all possible assignments on $I_{p'}$.
	
	The third and final stage consists of a one polynomial time round. For each relevant bin $y_{t_j}$, the algorithm uses the result of the second phase to find an assignment that $y_{t_j}$ is sensitive w.r.t.~it. Then, the algorithm discovers the original relevant variable $x_{i_k} \in D$ that was mapped to $y_{t_j}$ using Lemma~\ref{binarySearch} part \ref{prt:2}. This is a non-adaptive procedure that runs in polynomial time and uses $O(\log{n})$ queries. This is done in parallel for all $y_{t_j} \in I_{p'}$. Then, the algorithm can output the original function.
	If the first stage succeeds, this is the correct target function. Therefore, the success probability of the algorithm is $1-\delta$.
	
	In total, for $\delta \leq 1/d$, the number of rounds of this algorithm is $O(d\log{d})$, the number of queries is
	$O\left(d\log{n}+2^d\log{1/\delta}\right)$ and the time complexity is $poly(2^d,n,\log{(1/\delta)})$.

	For $\delta > 1/d$, the algorithm performs basically the same three stages with the following simple two changes. In the first stage, one partition into $d^3$ variables is enough. In addition, the first stage is performed only one time without repetitions. The number of rounds does not change. Neither does the time complexity. Analyzing the new query complexity for $\delta > 1/d$ gives
	$$
	O\left(d\log{n}+ 2^d(\ln(2d)+\ln(2/\delta))+d\log{d}+2^d  \right) = O(d\log{n}+(\log{d})2^d) \enspace.
	$$
\end{proof}


The results of the algorithms in Theorems~\ref{randAdap1}, \ref{poly2round} and \ref{manyRounds} are summarized in the introduction in Table~\ref{AdaptiveTable} for adaptive learning. One can choose between them depending on the requested value of $\delta$ and the wanted number of rounds.

For randomized adaptive learning, we have two upper bounds. First, the above algorithm from Theorem~\ref{manyRounds}. Second, the deterministic non-adaptive algorithm from Theorem~\ref{Theo2} that asks $O(d2^d\log{n})$ queries.
In addition, the deterministic non-adaptive algorithm from Theorem~\ref{Theo2} along with any other algorithm presented in this section of randomized adaptive learning, they all satisfy the time complexity of $n^{O(d)}$.
Therefore, in this case, the selection between them should be based on the requested value of $\delta$ and the restrictions, if exist, on the number of rounds.

\section{Conclusions and Future Work}
The problem of exact learning of $d$-Junta from membership queries is a fundamental problem with many applications, as discussed in the introduction. Therefore, it would be interesting to try to close or tighten the gaps between the lower and upper bounds in the different presented settings. 

We though note here that any further improvement in the deterministic non-adaptive upper bound would be a breakthrough result since it would give a new upper bound result for the size of an $(n,d)$-universal set, an important combinatorial longstanding open problem. This is true since, as mentioned before, any set of assignments that learns $d$-Junta must be an $(n,d)$-universal set, and $O(d2^d\log{n})$ is the best known upper bound for the size of $(n,d)$-universal sets.

In addition, while using the minimax technique for this problem has resulted in a new lower bound for randomized non-adaptive learning, we believe that to improve the lower bounds presented in this paper one needs to use a new technique that is beyond the existing ones. 

Improving some of the results presented in this paper would be a direct outcome of improvements on the underlying used structures. We find it in particular intriguing to try to improve the $poly(2^d,n)$ time deterministic construction of $(n,d)$-universal sets from Lemma \ref{LemmaUS}. This would give new non-adaptive and adaptive results.

Future work related to the topic of learning $d$-Junta include the study of monotone $d$-Junta. It might be fruitful to revisit the problem of monotone $d$-Junta in light of the work done in this paper. Another problem worth reexamining is the study of decision trees problem, a natural generalization of the $d$-Junta problem.




\bibliography{general}
\bibliographystyle{elsarticle-num}

\end{document}